\documentclass[times,twocolumn,final,authoryear]{elsarticle} %{elsarticle}

\usepackage{framed,multirow}

\usepackage{amssymb}
\usepackage{latexsym}

\usepackage{url}
\usepackage{xcolor}
\definecolor{newcolor}{rgb}{.8,.349,.1}

\usepackage{amsthm}
\usepackage{graphicx}
\usepackage{subfigure}

\newcommand{\be}{\begin{equation}}
\newcommand{\ee}{\end{equation}}

\newcommand{\sset}[1]{\left\{{#1}\right\}}

\newcommand{\prob}[1]{\hbox{Pr}\left\{{#1}\right\}}

\newcommand{\vmin}{v_{*}} %{v_{\hbox{min}}}

\newtheorem{theorem}{Theorem}

\journal{Pattern Recognition Letters}

\begin{document}

\ifpreprint
  \setcounter{page}{1}
\else
  \setcounter{page}{1}
\fi

\begin{frontmatter}

\title{Selecting a number of voters for a voting ensemble}

\author[1]{Eric Bax} 
\ead{ebax@verizonmedia.com}

\address[1]{ebax@verizonmedia.com\\Verizon Media, Playa Vista, CA, USA}

\begin{abstract}
For a voting ensemble that selects an odd-sized subset of the ensemble classifiers at random for each example, applies them to the example, and returns the majority vote, we show that any number of voters may minimize the error rate over an out-of-sample distribution. The optimal number of voters depends on the out-of-sample distribution of the number of classifiers in error. To select a number of voters to use, estimating that distribution then inferring error rates for numbers of voters gives lower-variance estimates than directly estimating those error rates. 
\end{abstract}

\end{frontmatter}

%\linenumbers

%% main text
\section{Introduction}
Voting ensembles of classifiers are a staple of machine learning, including bagging \citep{breiman96}, boosting \citep{schapire90,freund97}, forests of decision trees \citep{ho95,ho98,breiman01}, and stacking \citep{wolpert92}. Comparative studies show that ensemble classifiers are often the best types of out-of-the-box classifiers \citep{li10}, they win many machine learning competitions \citep{bennett07,chen16}, and they continue to solve practical problems \citep{zheng19,jowua13}. For a compelling explanation why ensemble classifiers produce good performance, refer to \citep{dietterich00}. For more on selecting ensemble classifiers and ensemble size, refer to \citep{bonab19,jackowski18,gomes17,lobato13,oshiro12,yang11,rokach10,rokach09,tsoumakas08,kuncheva04,kuncheva03,liu04,hu01,bax_voting,lam97}.
%For more on methods to select classifiers for a voting ensemble, refer to \cite{bax_voting}

This paper focuses on equally-weighted voting. One extreme is to use all classifiers as voters for every classification. The other is to select a single classifier at random for each classifier. This is sometimes called Gibbs classification. PAC-Bayes error bounds \citep{mcallester99,langford01,begin16} indicate why Gibbs classification can be effective -- selecting a Gibbs ensemble that includes 1\% of the hypothesis classifiers produces error bounds that are similar to selecting a classifier from a hypothesis set with only 100 classifiers, even if the actual hypothesis set has an arbitrarily large number of classifiers. 

Similar to \cite{esposito04}, this paper explores how to select a number of voters for a majority-vote ensemble classifier. That paper shows that the distribution of the number of classifier errors can vary widely over examples in empirical datasets, motivating our analysis of the influence of ensemble size on error rates in such situations. That paper focuses on ensembles of classifiers selected with replacement, allowing a potentially unlimited number of voters. In contrast, this paper focuses on ensembles of classifiers selected without replacement, leading to a different conclusion about the optimal ensemble size, and a statistical method to select the ensemble size based on simultaneous bounds on frequencies of the numbers of classifiers in error. 

The next section shows that selecting a single classifier at random for each example can outperform voting over all ensemble classifiers. In Section \ref{section_basis}, we show how the distribution of number of ensemble classifiers in error impacts the optimal number of voters, by considering the error curves over numbers of voters for each number of classifier errors as a basis for all possible distribution error curves over numbers of voters. In Section \ref{section_optimal}, we prove that any number of voters may be optimal. Section \ref{section_select} discusses methods to select the number of voters, showing that an out-of-sample error estimate based on inference has less variance than direct estimates. Then Section \ref{section_inf_val} outlines methods to compute out-of-sample error bounds based on inference estimates, including a discussion of challenges for the future. %Section \ref{section_select} gives advice on selecting a number of voters based on data. Section \ref{section_impact} discusses how selecting representatives at random can impact society. 

\section{Does voting always help?}
To begin, compare worst-case error rate for majority voting over all classifiers to selecting a single classifier at random for each example. Let $m$ be the number of classifiers in the ensemble, and let $p$ be their average error rate. Then the single-classifier strategy has error rate $p$, by linearity of expectations. For the all-voting strategy, the out-of-sample error rate depends on patterns of agreement among the classifiers as well as their out-of-sample error rates. For example, suppose three classifiers each have a 10\% error rate. Then, for each pair of classifiers, it is possible that they err together (and the classifier outside the pair is correct) with probability 5\%, producing a 15\% voting error rate. In general,

\begin{theorem}
For $m$ classifiers in an ensemble, with $m$ odd, if the average out-of-sample classifier error rate is $p$, with $p < \frac{1}{2}$, then the maximum possible all-voting error rate is
\be 2 p \left(1 - \frac{1}{m+1}\right). \ee
\end{theorem}

\begin{proof}
All-voting error is maximized by having the smallest possible majority of voters in error be as probable as the average error bound allows, and otherwise having zero errors.
Maximize the probability (call it $\hat{p}$) of the slimmest majority, $\frac{m+1}{2}$, being incorrect, given the constraint that the sum of error rates over classifiers is $mp$:
\be \hat{p} \left(\frac{m+1}{2}\right) = mp, \ee
and solve for $\hat{p}$:
\be \hat{p} =  2 p \left(1 - \frac{1}{m+1}\right). \ee
If $\frac{m+1}{2}$ are incorrect with probability $\hat{p}$ and zero are incorrect with probability $1 - \hat{p}$, then $\hat{p}$ is the all-voting error rate.
\end{proof}

As $m$ increases, all-voting error rate can approach twice the error rate of selecting a classifier at random, because a voting classifier can have nearly half its classifiers correct and still be incorrect. Discretization can favor the individual.

\section{A basis to analyze voting} \label{section_basis}
Now consider ensemble classifiers with (odd) $m$ classifiers, average classifier error rate $p$, and an (odd) number of voters $v$ from one to $m$. Setting $v=1$ gives the single-voter strategy, and $v=m$ gives the all-voting strategy. In this section, we analyze error rate curves over numbers of voters, given numbers of ensemble classifiers incorrect. Those curves form a basis for error rate curves over numbers of voters, in the sense that the weighted sums of those basis curves (with nonnegative weights that sum to one) are all the possible error rate curves over numbers of voters:

\begin{theorem} \label{subset_error_rate}
For an ensemble with $m$ classifiers, for each $i \in \sset{0, \ldots, m}$, let $w_i$ be the probability that exactly $i$ of the $m$ classifiers are in error for an example drawn at random from an out-of-sample distribution. Then, for subset voting with $v$ classifiers, the average error rate over the out-of-sample distribution is
\be \sum_{i=0}^{m} w_i r(m, i, v), \ee
where
\be r(m, i, v) = \sum_{j = \frac{v+1}{2}}^{v} \frac{{{i}\choose{j}} {{m-i}\choose{v-j}}}{{{m}\choose{v}}} \label{r_def} \ee
is the expected voting error rate given that $i$ of $m$ classifiers are in error.
\end{theorem}

\begin{proof}
Voting has average error rate
\be \sum_{i=0}^{m} w_i \prob{\hbox{voting error}|i}, \ee
where $i$ is the number of the $m$ classifiers that are in error. Voting error requires a majority of voters to be in error, so
\be \prob{\hbox{voting error}|i} = \sum_{j = \frac{v+1}{2}}^{v} \prob{\hbox{$j$ voters are in error} | i}, \ee
and
\be \prob{\hbox{$j$ voters are in error} | i} = \frac{{{i}\choose{j}} {{m-i}\choose{v-j}}}{{{m}\choose{v}}}, \ee
since this is the number of ways to select $j$ voters from the $i$ incorrect ones and $v-j$ voters from the $m-i$ correct ones, divided by the number of ways to select $v$ voters from the $m$ classifiers.
\end{proof}

\begin{figure} 
\includegraphics[width=3.5in]{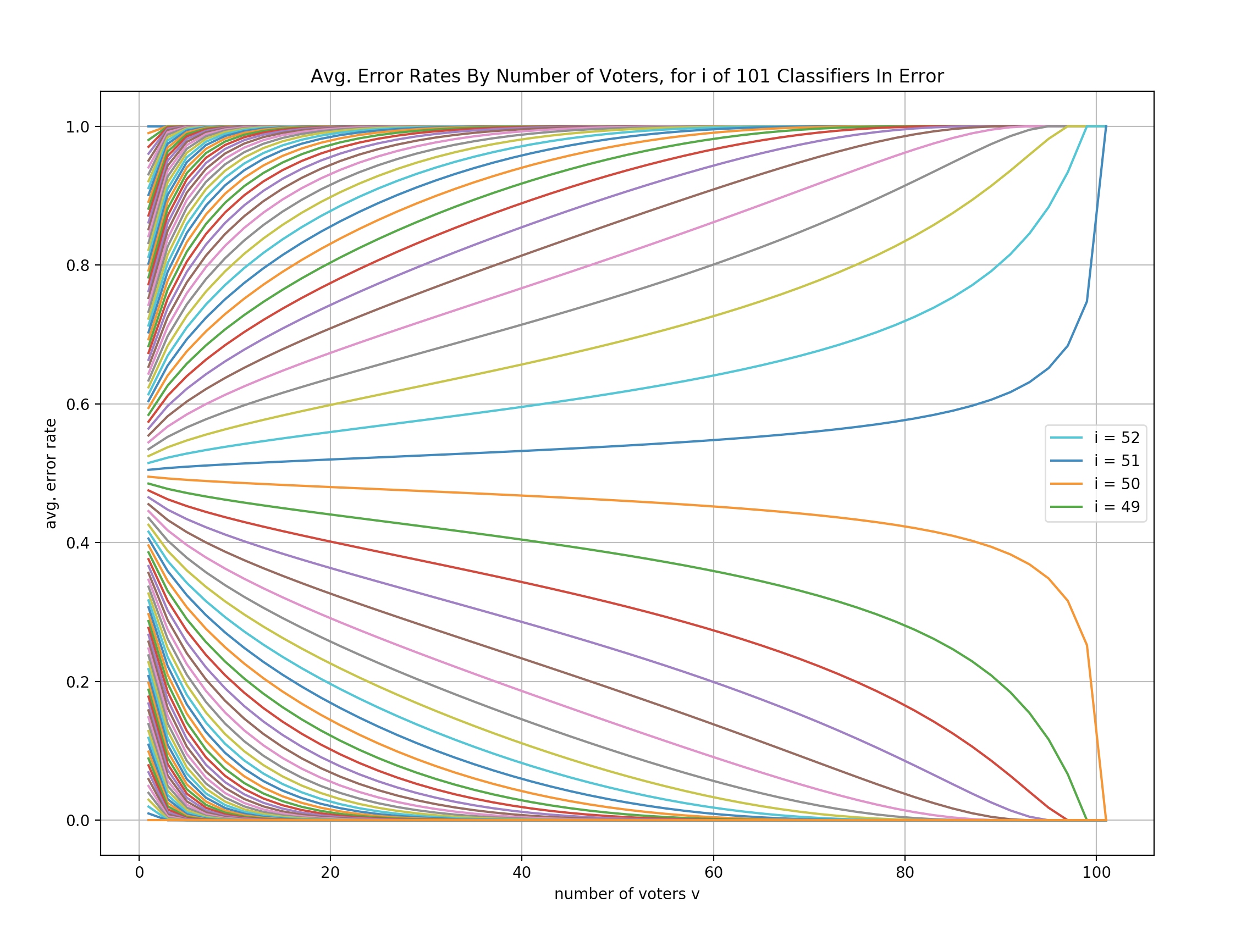} 
\caption{Average error rates $r(m, i, v)$, for $v$ voters, if $i \in \sset{0, \ldots, 101}$ of $m = 101$ classifiers are incorrect. The line for $i = 0$ is at the bottom of the figure -- a flat line of average error rate zero, since if all classifiers are correct, then any number of voters gives correct classification. Similarly, the line for $i = 101$ is at the top, and successive curves are for successive $i$-values. (Only the middle four are labeled, to avoid cluttering the figure.).}
\label{basis_plot}
\end{figure}

Figure \ref{basis_plot} shows how the probability of error for a voting classifier, $r(m, i, v)$ varies with the number of voters $v$, given that $i$ of $m$ classifiers are in error. The plot is for $m = 101$, with a curve for each number of errors $i \in \sset{0, \ldots, 101}$, connecting $r(m, i, v)$ values for odd $v$ from one to 101. Think of the curves as a basis. The set of weighted averages of the curves is the set of all possible error curves with respect to number of voters $v$ for ensembles of $m = 101$ classifiers.

From Figure \ref{basis_plot}, notice that if more than half the classifiers are in error ($i \geq 51$), then using more voters increases the out-of-sample error rate (except for $i = 101$ since it always gives a 100\% error rate). To see why consider $i = 51$. If 51 of 101 classifiers are in error, then selecting a single one and using it gives expected error rate $\frac{51}{101}$, but using all $101$ always results in an error, because 51 is a (slight) majority of the 101 voters. Between 1 and 101, using more voters increases the error rate, by making it more likely that the majority of voters will be incorrect. Also, notice that as $i$ increases above 51, the increase in error rate with number of voters  goes from concave up to nearly linear to concave down, showing that the loss in accuracy due to using a few voters rather than a single classifier increases with $i$. 

The following theorem shows that some things we can observe in Figure \ref{basis_plot} are true in general: adding voters strengthens classification for examples with fewer than half the ensemble classifiers in error, and it weakens classification for examples with more than half the ensemble classifiers in error. These effects only cease when there are so many voters that the minority among the ensemble (correct or incorrect) is too small to be a majority of the voters.

\begin{theorem} \label{updown_thm}
Let $r(m, i, v)$ be the voting ensemble error rate if $i$ of $m$ ensemble classifiers are in error, $v$ voters are selected at random without replacement from the ensemble, and their majority vote is returned. (Assume $v$ is odd.) Define the error rate difference due to increasing the number of voters by two:
\be\Delta_v(m, i, v) \equiv r(m, i, v + 2) - r(m, i, v).\ee
Then:
\begin{itemize}
\item If $i < \frac{v+1}{2}$ then $r(m, i, v) = 0$ and $\Delta_v(m, i, v) = 0$.
\item If $i<\frac{m}{2}$ and $i \geq \frac{v+1}{2}$, then $\Delta_v(m, i, v) < 0$.
\item If $i>\frac{m}{2}$ and $m - i \geq \frac{v+1}{2}$, then $\Delta_v(m, i, v) > 0$.
\item If $m - i < \frac{v+1}{2}$ then $r(m, i, v) = 1$ and $\Delta_v(m, i, v) = 0$.
\end{itemize} 
\end{theorem}

\begin{proof}
The first and last bullet points in the theorem are straightforward. For the others, suppose $v$ voters have been selected from $m$ classifiers. How can adding two voters make an incorrect decision correct? There is only one way: the $v$ voters must contain the maximum number of incorrect voters to still be correct: $\frac{v-1}{2}$, and both added voters must be incorrect. That makes the number of incorrect voters
\be \frac{v-1}{2} + 2 = \frac{v + 3}{2} = \frac{(v+2) + 1}{2}, \ee
out of $v+2$ voters. Since this is the minimum possible majority, starting with fewer than $\frac{v-1}{2}$ incorrect voters or adding fewer than two incorrect voters will not work. Similarly, to go from incorrect with $v$ voters to correct with $v+2$ voters, it is necessary to start with $\frac{v+1}{2}$ incorrect voters and add two correct voters. 

Note that $\Delta_v(m, i, v)$ is the difference between the probability of going from correct to incorrect and the probability of going from incorrect to correct. The probability of going from incorrect to correct is the probability of selecting $\frac{v-1}{2}$ of the $i$ incorrect voters when selecting $v$ of the $m$ classifiers without replacement, times the probability of getting two incorrect classifiers when selecting two of the remaining $m-v$ classifiers as voters:
\be \frac{{{i}\choose{\frac{v-1}{2}}} {{m-i}\choose{\frac{v+1}{2}}}}{{{m}\choose{v}}} \left(\frac{i - \frac{v-1}{2}}{m-v}\right) \left(\frac{i - \frac{v-1}{2} - 1}{m - v- 1}\right). \ee
Similarly, the probability of going from incorrect to correct is:
\be \frac{{{i}\choose{\frac{v+1}{2}}} {{m-i}\choose{\frac{v-1}{2}}}}{{{m}\choose{v}}} \left(\frac{m - i - \frac{v-1}{2}}{m-v}\right) \left(\frac{m - i - \frac{v-1}{2} - 1}{m - v- 1}\right).\ee

To find the difference, $\Delta_v(m, i, v)$, note that
\be {{i}\choose{\frac{v-1}{2}}} = \frac{i (i-1)!}{\left(\frac{v-1}{2}\right)! \left(i - \frac{v-1}{2}\right) \left(i - 1 - \frac{v-1}{2}\right)!} 
\ee
\be
= i {{i-1}\choose{\frac{v-1}{2}}} \left(\frac{1}{i - \frac{v-1}{2}}\right), \ee
and
\be {{i}\choose{\frac{v+1}{2}}} = \frac{i (i-1)!}{\left(\frac{v+1}{2}\right) \left(\frac{v-1}{2}\right)! \left(i - \frac{v+1}{2}\right)!} 
\ee
\be
= \frac{i (i-1)!}{\left(\frac{v+1}{2}\right) \left(\frac{v-1}{2}\right)! \left(i - 1 - \frac{v-1}{2}\right)!} 
\ee
\be
= i {{i-1}\choose{\frac{v-1}{2}}} \left(\frac{1}{\frac{v+1}{2}}\right). 
\ee
Similarly, 
\be {{m-i}\choose{\frac{v+1}{2}}} = \left(m - i\right) {{m-i-1}\choose{\frac{v-1}{2}}} \left(\frac{1}{\frac{v+1}{2}}\right),\ee
and
\be  {{m-i}\choose{\frac{v-1}{2}}} = \left(m - i\right) {{m-i-1}\choose{\frac{v-1}{2}}} \left(\frac{1}{m - i - \frac{v-1}{2}}\right).\ee
Use these equalities to factor out some common terms:
\be \Delta_v(m, i, v) = {{m}\choose{v}}^{-1} i {{i-1}\choose{\frac{v-1}{2}}} \left(m - i\right) {{m-i-1}\choose{\frac{v-1}{2}}} \ee
\be \left[ \frac{\left(\frac{i - \frac{v-1}{2}}{m-v}\right) \left(\frac{i - \frac{v-1}{2} - 1}{m - v- 1}\right)}{\left(i - \frac{v-1}{2}\right) \left(\frac{v+1}{2}\right)}  \right.\ee
\be \left. -  \frac{\left(\frac{m - i - \frac{v-1}{2}}{m-v}\right) \left(\frac{m - i - \frac{v-1}{2} - 1}{m - v- 1}\right) }{\left(\frac{v+1}{2}\right) \left(m - i - \frac{v-1}{2}\right)} \right].\ee
For both terms in brackets, cancel one denominator factor with a numerator factor, and factor out the other denominator factors:
\be \Delta_v(m, i, v) = {{m}\choose{v}}^{-1} i {{i-1}\choose{\frac{v-1}{2}}} \left(m - i\right) {{m-i-1}\choose{\frac{v-1}{2}}} 
\ee
\be
\frac{\left[ \left(i - \frac{v-1}{2} - 1\right) - \left(m - i - \frac{v-1}{2} - 1\right)\right]}{\left(\frac{v+1}{2}\right) \left(m-v\right) \left(m-v-1\right)} .\ee
Then cancel in brackets:
\be \Delta_v(m, i, v) = {{m}\choose{v}}^{-1} i {{i-1}\choose{\frac{v-1}{2}}} \left(m - i\right) {{m-i-1}\choose{\frac{v-1}{2}}} \ee
\be \frac{\left[2i - m\right]}{\left(\frac{v+1}{2}\right) \left(m-v\right) \left(m-v-1\right)}.\ee
Only the term $2i-m$ may be negative. It is negative if $i < \frac{m}{2}$ and positive if $i > \frac{m}{2}$. Based on the second through fifth factors, $\Delta_v(m, i, v) = 0$ if at least one of the following conditions holds: $i = 0$, $i = m$,  $i - 1 < \frac{v-1}{2}$ (equivalently: $i < \frac{v+1}{2}$), or $m - i < \frac{v+1}{2}$.
\end{proof}

\section{Optimal numbers of voters} \label{section_optimal}
\begin{figure} 
\includegraphics[width=3.5in]{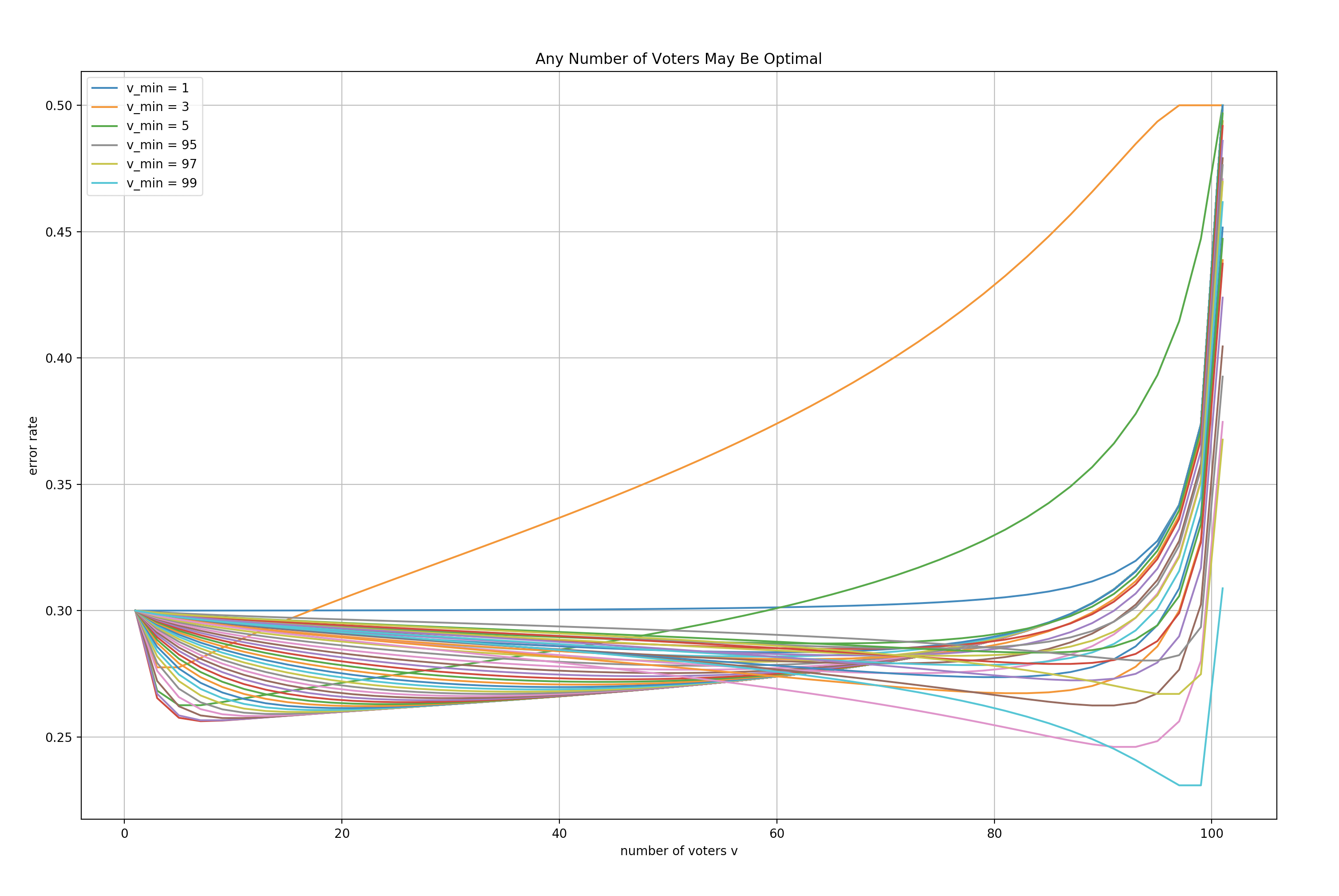} 
\caption{Any (odd) number of voters may be optimal. For $m = 101$, error curves for which each $\vmin \in \sset{1, 3, \ldots, 99}$ is an error-minimizing number of voters. (Each curve is based on a different distribution over number of ensemble classifiers incorrect. Average ensemble classifier error rate is 0.3.)}
\label{gap_plot}
\end{figure}

\begin{figure}
\includegraphics[width=3.5in]{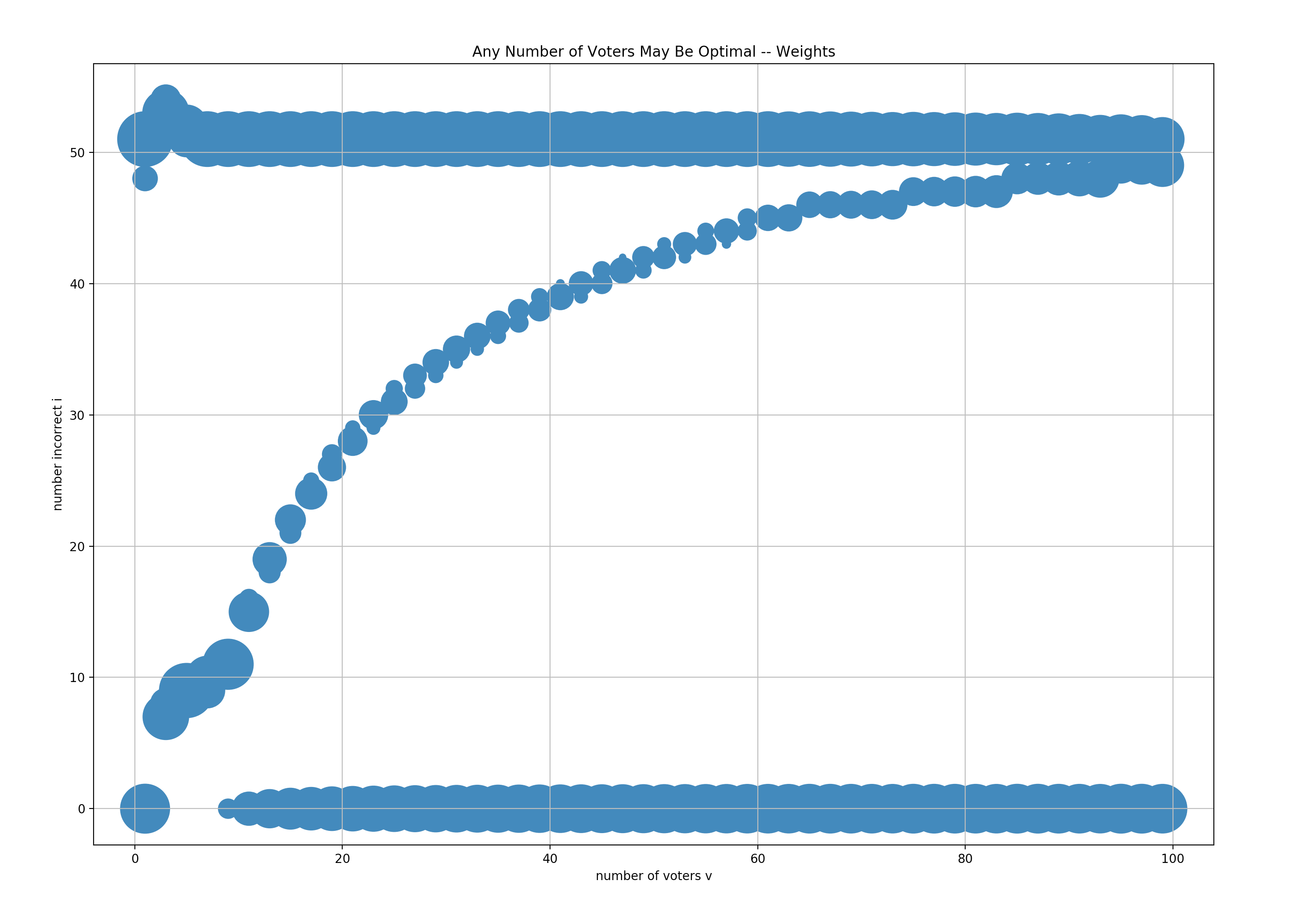} 
\caption{Distributions $w_i$ for $i \in \sset{0, \ldots, 101}$ that produce curves for which $v = \vmin$ is an error-minimizing number of voters.}
 \label{gap_plot_weights}
\end{figure}

Any (odd) number of voters may be optimal. For each $\vmin \in \sset{1, 3, \ldots, 99}$, with $m = 101$, Figure \ref{gap_plot} shows a voting error rate curve for a distribution $(w_0, \ldots, w_m)$ for which $\vmin$ is an error-minimizing number of voters. For each curve, the distribution $(w_0, \ldots, w_m)$ is produced by a linear program that maximizes difference in error rates between using 101 and using $\vmin$ voters, subject to constraints:
\begin{itemize}
\item voting error rate with $\vmin$ classifiers is no more than that for a single voter, for $\vmin-2$ voters, or for $\vmin+2$ voters,
\item voting error rate with all classifiers voting is at most $0.5$ (otherwise, it is possible to just use the opposite of its output to get a lower error rate), and
\item average error rate over classifiers is 0.3.
\end{itemize}

The figure shows that any odd number of voters up to 99 can minimize error rate for $m = 101$. The figure does not show this for $\vmin = 101$, since the linear program optimizes the difference between using $\vmin$ and using 101 voters. However, Figure \ref{basis_plot} shows that 101 is the optimal number of voters for the distribution $w_{50} = 1$ and all other $w_i = 0$. 

Figure \ref{gap_plot_weights} shows the distributions of weights for each $\vmin$ curve in Figure \ref{gap_plot}. For most $\vmin$ values, the error count distributions $(w_0, \ldots, w_m)$ that give the largest gap between using $\vmin$ voters and the full set of $m = 101$ voters have some weight on $w_0$, some on $w_{51}$, and some on one or two intermediate numbers of errors. The curve for each $\vmin$ is the weighted sum of the "basis" curves for $i=0$, $i=51$, and the intermediate value or values from Figure \ref{basis_plot}. The weight on $w_0$ lowers error rates, helping enforce the constraint on average error rate over classifiers and the constraint on error rate for 101 voters. The weight on $w_{51}$ ensures that the curve for $\vmin$ goes up on the right, increasing the difference in error rates between $\vmin$ voters and 101 voters as well as helping to ensure that $\vmin$ voters is locally optimal compared to $\vmin + 2$ voters. Any weights on intermediate values can contribute to the curvature, ensuring that $\vmin$ voters is locally optimal compared to $\vmin - 2$ voters. These weights can give us insight into how to prove that any (odd) number of voters can be optimal for any number $m$ of classifiers: 

\begin{theorem} \label{any_number_thm}
For $m > 0$ and any $\vmin \in \sset{1, 3, \ldots, m}$, there is a distribution $(w_0, \ldots, w_m)$, where $w_i$ is the probability that $i$ of $m$ classifiers are in error, such that $\vmin$ voters achieves the minimum voting error rate.
\end{theorem}

\begin{proof}
First consider the boundary cases: $\vmin = 1$ and $\vmin = m$ (or $\vmin = m - 1$ if $m$ is even). If $\vmin=1$, then set $w_{m-1} = 1$. That gives a $\frac{m-1}{m}$ error rate for $v = 1$, since there is one correct classifier, and error rate one for $v \geq 3$, since the single correct classifier cannot form a majority. Similarly, if $\vmin = m$ or $\vmin = m-1$, then set $w_i = 1$ for $i = \frac{\vmin - 1}{2}$. That produces error rate zero for $\vmin$ voters, because the incorrect voters cannot form a majority, but the error rate for smaller $v$ values is positive since the incorrect voters can form a majority for those values. 

Now consider intermediate values: $1 < \vmin < m - 1$. Let $a = \frac{\vmin-1}{2}$ and $b = m - \frac{\vmin+1}{2}$. For $a$ errors among $m$ classifiers, the voting error rate is zero for $v \geq \vmin$. For $b$ errors, it is one for $v > \vmin$, but less than one for $v = \vmin$, because the $\frac{\vmin+1}{2}$ can form a majority among $\vmin$ voters. As a result, any weighted average (with positive weights) of error curves for $a$ and $b$ has voting error rate greater than $\vmin$ for all $v > \vmin$. So set $w_a = \theta$, $w_b = 1 - \theta$, and $w_i = 0$ for all other $i$ values. 

Then select $\theta \in (0,1)$ to ensure that the linear combination of curves has voting error rates less than the rate for $\vmin$ for all $v < \vmin$. It is possible to do this by taking $\theta$ sufficiently close to one, because doing so makes the combined curve resemble the curve for $a$ errors more and the curve for $b$ errors less. (This requires that the error rates on the curve for $b$ be bounded, so that multiplying by $1 - \theta$ can reduce their influence as $1 - \theta$ approaches zero, but they are bounded by one since they are error rates.) By Theorem \ref{updown_thm}, the curve for $a$ errors is strictly decreasing in $v$ for $v < \vmin$, since $i = a = \frac{\vmin-1}{2} < \frac{m}{2}$ and $i = a = \frac{\vmin-1}{2}  \geq \frac{v+1}{2}$. 
\end{proof}

\section{Selecting the number of voters} \label{section_select}
Assume we have $n$ validation examples, drawn i.i.d. from the out-of-sample distribution, and not used to train or select classifiers for the ensemble, and we want to use them to select the number of voters for the ensemble. A straightforward method is to apply the voting classifier for each odd number of voters $v$ in one to $m$ to all the validation examples, selecting voters at random for each number of voters and validation example. Then calculate the error rate over the validation examples for each number of voters, and select the number of voters with the lowest validation error rate. 

To get error bounds, for each number of voters $v$, let $k_v$ be the number of errors over the validation examples. Let $u(n, k, \delta)$ be a PAC (probably approximately correct) upper bound and $t(n, k, \delta)$ be a PAC lower bound for the probability that produces $k$ events in $n$ samples with bound failure probability $\delta$. Since we simultaneously validate error rates for $\frac{m + 1}{2}$ different numbers of voters and use two-sided bounds, for $\delta > 0$, with probability at least $1 - \delta$, each number of voters has out-of-sample distribution error rate in the range $t(n, k_v, \frac{\delta}{m + 1})$ to $u(n, k_v, \frac{\delta}{m + 1})$. For example, using Hoeffding bounds \citep{hoeffding63}, the range is:
\be
\frac{k_v}{n} \pm \sqrt{\frac{\ln (m + 1) - \ln \delta}{2 n}}. \label{hoeff_direct}
\ee
In practice, use a binomial inversion bound \citep{hoel54,langford05}, to get a tighter bound range.

Selecting voters at random for each example introduces some variance. To produce lower-variance estimates of voting error rates, instead compute for each validation example $j$ the number of ensemble classifiers in error $i_j$. Recall from Equation \ref{r_def} that $r(m, i, v)$ is the expected voting error rate for $v$ voters, given $i$ ensemble classifiers in error. Then, for each number of voters $v$, an estimate of out-of-sample error rate is:
\be
\frac{1}{n} \sum_{j=1}^{n} r(m, i_j, v). \label{inf_est_sum}
\ee
Let $\hat{w}_i$ be the fraction of validation examples for which there are $i$ classifiers in error ($i_j = i$). Then the estimate is
\be
= \sum_{i = 0}^{m} \hat{w}_i r(m, i, v). \label{better_est}
\ee
Refer to these estimates as the inference estimates, since they use estimates of the rates of numbers of errors in the ensemble to infer a estimates of out-of-sample error rates for different numbers of voters. 

\begin{theorem}
For each number of voters $v$, the out-of-sample error rate estimate from using validation data to compute estimates $\hat{w}_i$ and using them to infer average voting error rate has variance less than or equal to the estimate from applying randomly-selected voters to validation examples.
\end{theorem}

\begin{proof}
For inference, the estimate is the mean of $r(m, i, v)$ over validation examples, where $i$ is the number of ensemble classifiers in error. For brevity, call $r(m, i, v)$ simply $r_i$. The direct estimate is the mean over validation examples of one if voting errs and zero otherwise. For both, the variance sums over examples, since the validation examples are independent samples. So we can focus on the variances for a single validation example. 

For direct estimation, the single-example error estimate is one with probability $p$, where $p$ is the out-of-sample distribution error rate for using $v$ voters, and zero with probability $1 - p$. Since it is a Bernoulli random variable, it has variance $p (1 - p) = p - p^2$. 

For inference, the single-example error estimate is $r(m, i, v)$, with probability $w_i$ for each value of $i$. The variance is the difference between the expectation of the square and the square of the expectation:
\be
\sum_i w_i r_i^2 - \left[ \sum_i w_i r_i \right]^2.
\ee
But 
\be
p = \sum_i w_i r_i,
\ee
so the variance is
\be
\sum_i w_i r_i^2 - p^2.
\ee
Recall that the variance for direct estimation is $p - p^2$. So the difference between variances is
\be
p - \sum_i w_i r_i^2  = \sum_i w_i r_i - \sum_i w_i r_i^2 = \sum_i w_i \left[r_i - r_i^2\right].
\ee
Since each $r_i$ is a probability (of voting error given $i$ ensemble errors), $r_i \in [0, 1]$. So $r_i^2 \leq r_i$. 
\end{proof}

\section{Error Bounds and Inference} \label{section_inf_val}
\subsection{Error Bounds}
Now consider how to compute error bounds based on the inference estimates. One way is to apply simultaneous Hoeffding bounds:

\begin{theorem} \label{inf_bd_thm}
With probability at least $1 - \delta$, the out-of-sample error rates for all numbers of voters $v$ are within 
\be
\sqrt{\frac{\ln (m + 1) - \ln \delta}{2 n}}
\ee
of the inference estimates from Expressions \ref{inf_est_sum} and \ref{better_est}.
\end{theorem}

\begin{proof}
From Expression \ref{inf_est_sum}, the inference estimate for each $v$ is
\be
\frac{1}{n} \sum_{j=1}^{n} r(m, i_j, v).
\ee
This is the average of $n$ i.i.d. random variables $r(m, i_j, v) \in [0, 1]$, one for each validation example $j$. The estimate's mean (over draws of validation examples) is the out-of-sample error rate for $v$, since $r(m, i, v)$ is the average error rate for $v$ voters, averaged over all sets of $v$ voters. So Hoeffding bounds apply \citep{hoeffding63}. There are $\frac{m + 1}{2}$ numbers of voters, and two-sided bounds, making $m + 1$ simultaneous validations. So set $\delta$ to $\frac{\delta}{m + 1}$ in the Hoeffding bound.
\end{proof}

\subsection{Validation by Inference}
Alternatively, we could first bound out-of-sample rates of $i$ errors among the ensemble classifiers, then use those bounds to infer bounds on ensemble error rates. (This strategy is called validation by inference \citep{bax_val_by_inference}.) As before, let $w_i$ be the out-of-sample rate of $i$ errors among the $m$ ensemble classifiers, and let $\hat{w}_i$ be the in-sample rate over $n$ validation examples. Also as before, let $u(n, k, \delta)$ and $t(n, k, \delta)$ be PAC upper and lower bounds on the out-of-sample probability that produces $k$ events in $n$ samples, with bound failure probability $\delta$. 

Then, using simultaneous validation for $2 (m + 1)$ bounds (2 for upper and lower bounds, and $m + 1$ for zero to $m$ errors among the ensemble classifiers), with probability at least $1 - \delta$, 
\be
\forall_{i \in \{1, \ldots, m\}} w_i \in \left[t(n, \hat{w}_i n, \frac{\delta}{2 (m + 1)}), u(n, \hat{w}_i n, \frac{\delta}{2 (m + 1)}) \right]. \label{box_constraints}
\ee
(Dividing $\delta$ to get simultaneous bounds is known as the Bonferroni correction \citep{bonferroni36,dunn61}.) 

To derive out-of-sample error bounds for each number of voters $v$, use linear programming to optimize
\be
\sum_{i=0}^{m} w_i r(m, i, v)
\ee
over the set of feasible out-of-sample rates of numbers of errors given by Expression \ref{box_constraints}, with the additional constraints 
\be
\sum_{i=0}^{m} w_i = 1 \hbox{ and } \forall i: w_i \geq 0,
\ee
minimizing for lower bounds and maximizing for upper bounds.

Compare these bounds to the bounds from Theorem \ref{inf_bd_thm}. These bounds allow binomial inversion for each $t_i$ and $u_i$, since each $\hat{w}_i$ has a binomial distribution, and these tend to be tighter than Hoeffding bounds and other derived bounds. (Binomial inversion gives sharp bounds.) However, these bounds divide $\delta$ by $2(m + 1)$ rather than $m + 1$, doubling the bound failure probability allocated to each probability bound, because we use $m + 1$ estimated rates of numbers of errors among classifiers to infer error bounds for only $\frac{m+1}{2}$ numbers of voters $v$.

\subsection{Using the Multinomial Distribution -- Future Work}
Together, the $\hat{w}_i$ values are a sample from a multinomial distribution with probabilities $w_i$. We should be able to use this information to get tighter bounds on the $w_i$ values from multinomial bounds, instead of using simultaneous binomial bounds over all $\hat{w}_i$. Let $L(\mathbf{\hat{w}}, \delta)$ be the set of probability vectors $\mathbf{w} = (w_0, \ldots, w_m)$ that have probability at least $1 - \delta$ of generating a sample vector that is more likely than $\mathbf{\hat{w}} = (\hat{w}_0, \ldots, \hat{w}_m)$. Call $L(\mathbf{\hat{w}}, \delta)$ the likely set. (Informally, $\mathbf{w}$ is in the likely set unless $\mathbf{\hat{w}}$ is in the $\delta$-tail of the multinomial distribution generated by probabilities $\mathbf{w}$.) 

To compute an upper bound on out-of-sample error rate for each number of voters $v$, solve:
\be
\max_{\mathbf{w} \in L(\mathbf{\hat{w}}, \delta)} \sum_{i=0}^{m} w_i r(m, i, v).
\ee
(For lower bounds, minimize instead of maximizing.) This produces valid upper and lower bounds for all numbers of voters, with probability at least $1 - \delta$. 

Constraining $\mathbf{w}$ to the likely set based on the multinomial distribution allows validation by inference without using a Bonferroni correction to divide $\delta$ by $2 (m + 1)$ as was required for the constraints in Expression \ref{box_constraints}. This produces tighter constraints, so the resulting bounds are at least as strong.

The challenge lies in computing the error bounds. We want to optimize a linear function over the likely set. But the likely set may have a challenging shape for optimization, because the tails of multinomial distributions are not continuous in $\mathbf{w}$. To see why, suppose that for some $\mathbf{w}$ some other sample vector is equally as likely as $\mathbf{\hat{w}}$. Then perturbing $\mathbf{w}$ can change the ordering of those likelihoods, so that the likelihood of the other sample vector becomes greater than that of $\mathbf{\hat{w}}$, shifting the other sample vector's full probability out of the tail in a discontinuous jump. 

So one goal for future research is to identify a superset of the likely set that is amenable to optimization and does not contain distributions $\mathbf{w}$ outside the likely set that would significantly weaken the resulting error bounds. There are approximations to the likely set that make optimization easier. For example, consider Pearson's $X^2$ statistic \citep{pearson00,read88,cressie84}:
\be
X^2(\mathbf{w}, \mathbf{\hat{w}}) = (m + 1) \sum_{i=0}^{m} \frac{(\hat{w}_i - w_i)^2}{w_i},
\ee
with term $i$ zero if $w_i = \hat{w}_i = 0$. Asymptotically, $X^2$ has a chi-squared distribution with $m$ degrees of freedom. So if we let $C_{\delta}$ be the value for which the cdf of that distribution is $\delta$, then we can define an approximate likely set:
\be
\tilde{L}(\mathbf{\hat{w}}, \delta) \equiv \left\{\mathbf{w} \left| X^2(\mathbf{w}, \mathbf{\hat{w}}) \leq C_{\delta} \right. \right\}.
\ee
This set has smooth boundaries. However, it is not a superset of the likely set.

%For a looser approximation but a simpler set, we could use $X^2(\mathbf{\hat{w}}, \mathbf{w})$ in place of $X^2(\mathbf{w}, \mathbf{\hat{w}})$, forming an ellipsoid around $\mathbf{\hat{w}}$, allowing us to use quadratic programming (QP) to solve for approximate error bounds. The challenge would then be to bound the difference in error rates produced by the optimal $\mathbf{w}$ from the approximate likely set compared to the optimal $\mathbf{w}$ from the actual likely set. 

To identify a suitable superset of the likely set, perhaps we can apply bounds on the difference between $X^2$ (or similar statistics) and the chi-squared distribution \citep{matsunawa77,siotani84,bickel90,taneichi02,gaunt16,ouimet21} to expand an approximate likely set, for example by increasing the constraint $C_{\delta}$ on $X^2$ to ensure that all $\mathbf{w}$ in the likely set are enclosed in the region of the approximate set while keeping the set small enough to yield effective error bounds. It may also be possible to extend bounds on $L_1$ distance between $\mathbf{\hat{w}}$ and $ \mathbf{w}$ \citep{valiant17,balakrishnan18} to derive a superset of the likely set that gives linear constraints instead of a feasible region with curved boundaries. 

Within optimization procedures, computing tail probabilities (the cdf of $\mathbf{\hat{w}}$ given $\mathbf{w}$) directly is infeasible for even moderate-sized ensembles, because the number of ways for $n$ samples to fall into $m + 1$ categories is ${n+m}\choose{m}$. Monte Carlo methods \citep{hope68,jann08} can estimate the tail probability to arbitrary accuracy with arbitrarily high probability. Also, there are evolving methods to cleverly collect terms for exact tail computation \citep{baglivo92,keich06,resin20}.

The problem of identifying a worst likely generator for multinomial samples is an interesting statistical problem, with applications beyond ensemble validation. It is not clear whether approaches based on the multinomial tail can be more effective in practice than approaches that split $\delta$ and perform simultaneous validation over individual categories. (In our case, each number of errors $i$ among the ensemble classifiers is a category.) The simpler approach of simultaneous validation can use sharp binomial inversion bounds within each category, and it can also benefit from the freedom to split $\delta$ non-uniformly. Finally, both approaches could benefit from tuning the constraints to the weights (the $r(m, i, v)$ values in our case) used to evaluate the generating distributions. For example, it is possible to merge categories that have similar weights to reduce the problem's dimensionality. 

\bibliographystyle{model2-names}
\bibliography{bax}

\end{document}